\pdfoutput=1

\documentclass[11pt]{article}
\usepackage{tabularx}
\usepackage[preprint]{acl}
\usepackage{times}
\usepackage{latexsym}
\usepackage{booktabs}
\usepackage[T1]{fontenc}
\usepackage{comment}
\usepackage[utf8]{inputenc}

\usepackage{microtype}

\usepackage{inconsolata}
\usepackage{hyperref}       
\usepackage{url}    
\usepackage{graphicx}
\usepackage{amsmath,amssymb,amsfonts,amsthm}
\usepackage{listings}
\definecolor{mygray}{rgb}{0.5,0.5,0.5}
\definecolor{mygreen}{rgb}{0,0.6,0}
\definecolor{myorange}{rgb}{1,0.5,0}
\definecolor{mymauve}{rgb}{0.58,0,0.82}
\newtheorem{lemma}{Lemma}

\lstdefinestyle{mystyle}{
    backgroundcolor=\color{white},  
    commentstyle=\color{mygreen},
    keywordstyle=\color{blue},
    numberstyle=\tiny\color{mygray},
    stringstyle=\color{mymauve},
    basicstyle=\ttfamily\footnotesize,
    breakatwhitespace=false,         
    breaklines=true,                 
    captionpos=b,                    
    keepspaces=true,                 
    numbers=left,                    
    numbersep=5pt,                  
    showspaces=false,                
    showstringspaces=false,
    showtabs=false,                  
    tabsize=2
}

\lstdefinelanguage{Fortran}
  {keywords={program, subroutine, integer, implicit, none, real, if, else, end, do, then, mod, print, return},
   sensitive=true,
   comment=[l]{!},
   morecomment=[l]{!},
   morestring=[b]',
   alsoletter={.}
  }

\title{Enhancing Cross-Language Code Translation via Task-Specific Embedding Alignment in Retrieval-Augmented Generation}

\author{
  \textbf{Manish Bhattarai\textsuperscript{1}},
  \textbf{Minh Vu\textsuperscript{1}},
  \textbf{ Javier E. Santos \textsuperscript{2}},
\\
  \textbf{Ismael Boureima\textsuperscript{1}},
  \textbf{Daniel O' Malley \textsuperscript{2}},
\\
\\
  \textsuperscript{1}Theoretical Division, Los Alamos National Laboratory, Los Alamos, NM 87544, \\
  \textsuperscript{2} Earth \& Environmental Science Division,  Los Alamos National Laboratory, Los Alamos, NM 87544
\\
  \small{
    \textbf{Correspondence:} \href{mailto:ceodspspectrum@lanl.gov}{ceodspspectrum@lanl.gov}
  }
}

\begin{document}
\maketitle
\begin{abstract}
We introduce a novel method to enhance cross-language code translation from Fortran to C++ by integrating task-specific embedding alignment into a Retrieval-Augmented Generation (RAG) framework. Unlike conventional retrieval approaches that utilize generic embeddings agnostic to the downstream task, our strategy aligns the retrieval model directly with the objective of maximizing translation quality, as quantified by the CodeBLEU metric. This alignment ensures that the embeddings are semantically and syntactically meaningful for the specific code translation task. Our methodology involves constructing a dataset of 25,000 Fortran code snippets sourced from Stack-V2 dataset and generating their corresponding C++ translations using the LLaMA 3.1-8B language model. We compute pairwise CodeBLEU scores between the generated translations and ground truth examples to capture fine-grained similarities. These scores serve as supervision signals in a contrastive learning framework, where we optimize the embedding model to retrieve Fortran-C++ pairs that are most beneficial for improving the language model's translation performance. By integrating these CodeBLEU-optimized embeddings into the RAG framework, our approach significantly enhances both retrieval accuracy and code generation quality over methods employing generic embeddings. On the HPC Fortran2C++ dataset, our method elevates the average CodeBLEU score from 0.64 to 0.73, achieving a 14\% relative improvement. On the Numerical Recipes dataset, we observe an increase from 0.52 to 0.60, marking a 15\% relative improvement. Importantly, these gains are realized without any fine-tuning of the language model, underscoring the efficiency and practicality of our approach.
\end{abstract}

\section{Introduction}

Cross-language code translation is a critical task in modern software development, especially as legacy programming languages, such as Fortran, continue to be prevalent in scientific computing, while more contemporary languages like C++ are favored for their performance and versatility in production environments. The goal of automatic translation from Fortran to C++ is to preserve the functionality and structure of legacy code while benefiting from the optimizations and ecosystem of C++. However, achieving high-quality translations that adhere to the syntax and semantic norms of the target language remains a challenging problem, particularly when there is a lack of large, aligned datasets or evaluation metrics that cover both source and target languages effectively.

Traditional approaches to cross-language translation, such as Retrieval-Augmented Generation (RAG)~\citep{rag} typically involve two phases: first, retrieving relevant examples from a database, followed by a language model generating code conditioned on both the query and the retrieved examples. In prior efforts, the retrieval models in RAG systems have relied on general-purpose embedding models~\citep{bhattarai2024enhancing,liretrieval}, which are not tailored to the specific nuances of code translation. These embeddings aim to retrieve relevant pairs from the source and target languages but do not directly optimize for the quality of the generated code. As a result, while the retrieved examples may be relevant in a broad sense, they often fail to guide the language model towards producing translations that maximize fidelity to the ground truth in the target language.
This gap is particularly problematic in scenarios where explicit metrics, such as CodeBLEU~\cite{ren2020codebleu}-designed to assess both syntactic and semantic correctness of translated code—are only available for the target language (e.g., C++ in this case). Without aligning the retrieval mechanism to such a task-specific metric, the system may retrieve suboptimal examples, leading to poor code generation performance. The inability to leverage task-relevant quality metrics during retrieval weakens the overall system, limiting its effectiveness in high-accuracy code translation tasks.
To address these limitations, we propose a novel contrastive learning framework that aligns the retrieval phase of the RAG system with the goal of maximizing the CodeBLEU~\citep{codebert} score for the generated C++ code. We collect a dataset of 25,000 Fortran code examples from Stack V2~\citep{lozhkov2024starcoder} and use the LLaMA 3.1-8B~\citep{touvron2023llama} model to generate corresponding C++ translations. In the absence of ground truth C++ translations, we evaluate the quality of these translations using pairwise CodeBLEU similarity scores. This metric captures both syntactic correctness and semantic fidelity, providing a robust signal for aligning the retrieval model through contrastive learning.

The proposed approach aims to addresses the shortcomings of general-purpose embedding models by integrating task-specific metrics into the retrieval optimization process. By aligning the retrieval model with the downstream task of producing high-quality C++ code, our method ensures that the examples retrieved during inference are not just broadly similar but are semantically and syntactically aligned in a way that enhances the LLM's generative performance. The result is a significant improvement in translation quality, as measured by CodeBLEU, over previous methods that lack such alignment.

Our contribution is twofold: first, we demonstrate the effectiveness of contrastive learning for fine-tuning retrieval models in the context of cross-language code translation, using a task-specific metric to guide alignment. Second, we show that optimizing retrieval for downstream generation tasks can lead to state-of-the-art results, particularly in cases where aligned datasets are not readily available for both source and target languages. This work not only advances the field of code translation but also opens up new possibilities for applying similar techniques to other language pairs and domains where task-specific evaluation metrics are available for only one side of the translation.
\section{Related Work}

Historically, code translation strategies before the advent of LLMs relied heavily on rule-based and statistical machine translation (SMT) systems~\citep{smt}. These systems used predefined rules or statistical mappings between the source and target programming languages, such as tree-based translation approaches that mapped syntax trees between languages. While these methods provided structured and interpretable outputs, they were limited in their ability to handle the semantic complexities of different programming languages and struggled with code diversity, edge cases, and idiomatic translations.

With the rise of deep learning and LLMs, fine-tuning models on large datasets became the go-to method for improving code translation. Models like CodeBERT~\citep{codebert} and Codex~\citep{ccodex}, when fine-tuned on specific language pairs, improved translation quality by leveraging vast amounts of parallel code data. However, the main limitation of LLM fine-tuning lies in the resource-intensive process. Fine-tuning requires substantial amounts of labeled data and computational resources, making it impractical for niche or legacy languages like Fortran, where parallel data may be scarce.

As a next step, task-specific alignment of LLMs emerged to improve translation by better guiding the model's output. While alignment techniques help improve output fidelity, they still necessitate fine-tuning or explicit modification of the LLM itself, which can be resource-intensive and may still fall short of generalization when translating between languages with significant structural differences~\citep{mishra2024granite}.

RAG introduced a more flexible approach by allowing LLMs to retrieve and condition their outputs on example pairs from a relevant dataset. While RAG improves translation by augmenting the model’s input, the effectiveness of this strategy depends on the quality and relevance of the retrieved examples. In an example case \citep{bhattarai2024enhancing}, the retrieval step relies on general-purpose embeddings like Nomic-Embed or CodeBERT, which, although effective at retrieving semantically similar code, are not optimized for specific downstream metrics like CodeBLEU. As a result, the LLM might not always retrieve the examples that would best assist in producing translations aligned with target-specific quality metrics.

The approach we propose offers a significant advantage by focusing on semantic alignment of the retrieval mechanism without the need to fine-tune the LLM itself. Through contrastive learning, we optimize the embedding model to retrieve Fortran-C++ pairs that are more likely to maximize the downstream metric (e.g., CodeBLEU) when used by the LLM for generation. This strategy ensures that the most relevant examples are retrieved for each translation task, improving the generation quality without requiring computationally expensive fine-tuning of the LLM. This retrieval alignment makes RAG more efficient and better suited for translating between languages where high-quality paired datasets may not be available. By concentrating on improving the quality of retrieved examples, our method achieves high-quality translation with minimal additional model training, leveraging existing LLM capabilities more effectively.

\section{Methods} \label{sect:method}
This section provides the technical description of our proposed method.
\subsection{Problem setting} \label{subsect:setting}

 We consider the standard code translation scenario leveraging a language model \( G \), in which a target translated code \( c^t \) of a query source code \( c^s \) is generated using \( G\): 
\begin{equation}
    c^t = G\left(c^s \right) \label{eq:gen_base}
\end{equation}
In practice, conditioning \( G \) on \( k \) example pairs of source and target code \( D := \left\{ \left(c^s_i, c^t_i \right) \right\}_{i=1}^{k}\), can significantly enhance translation. This few-shot learning approach can be expressed as:
\[
c^t = G\left(c^s,  D \right)
\]

In a RAG framework, this process is further refined by integrating a retrieval mechanism \( R \) that identifies the most pertinent \( k \) example pairs from a large corpus \( \mathcal{C} \) based on the query \( c^s \). By expressing this retrieval step as $ D = R(c^s, \mathcal{C})$, we can describe the conventional translation scenario leveraging \( G \) as
\begin{equation}
c^t = G\left(c^s,  R(c^s, \mathcal{C}) \right) \label{eq:rag}
\end{equation}
In practice, the input source code are embedded using a neural network $\Psi$, which are generally agnostic to the downstream task. We denote $c^s_{\Psi}$ as the embedding of the source code $c^s$ under the embedding $\Psi$. Hence, Eq.~\ref{eq:rag} can be expressed as
\begin{equation}
    c^t = G\left(c^s,  R(c_\Psi^s, \mathcal{C}_\Psi) \right)
\end{equation}
under the usage of the embedding model $\Psi$. Here, the notation $ \mathcal{C}_\Psi$ refers to the fact that the embedding is applied onto the corpus of ${c}^s$.

Some common embedding modules for code translation are Nomic-Embed~\cite{nussbaum2024nomic}, StarEncoder~\citep{li2023starcoder}, and CodeBERT~\cite{codebert}. However, as the performance of the translation task heavily depends on the relevance and the alignment of the retrieved examples with respect to the query \( c^s \), as we will show in the following discussion, 
it is beneficial to optimize $\Psi$ for better code translation performance.  In this manuscript, we use notation $\Psi$ and $\Phi$ to refer to aligned and unaligned embedding modules, respectively.

\begin{figure*}[!ht]
    \centering
    \includegraphics[width=.7\linewidth]{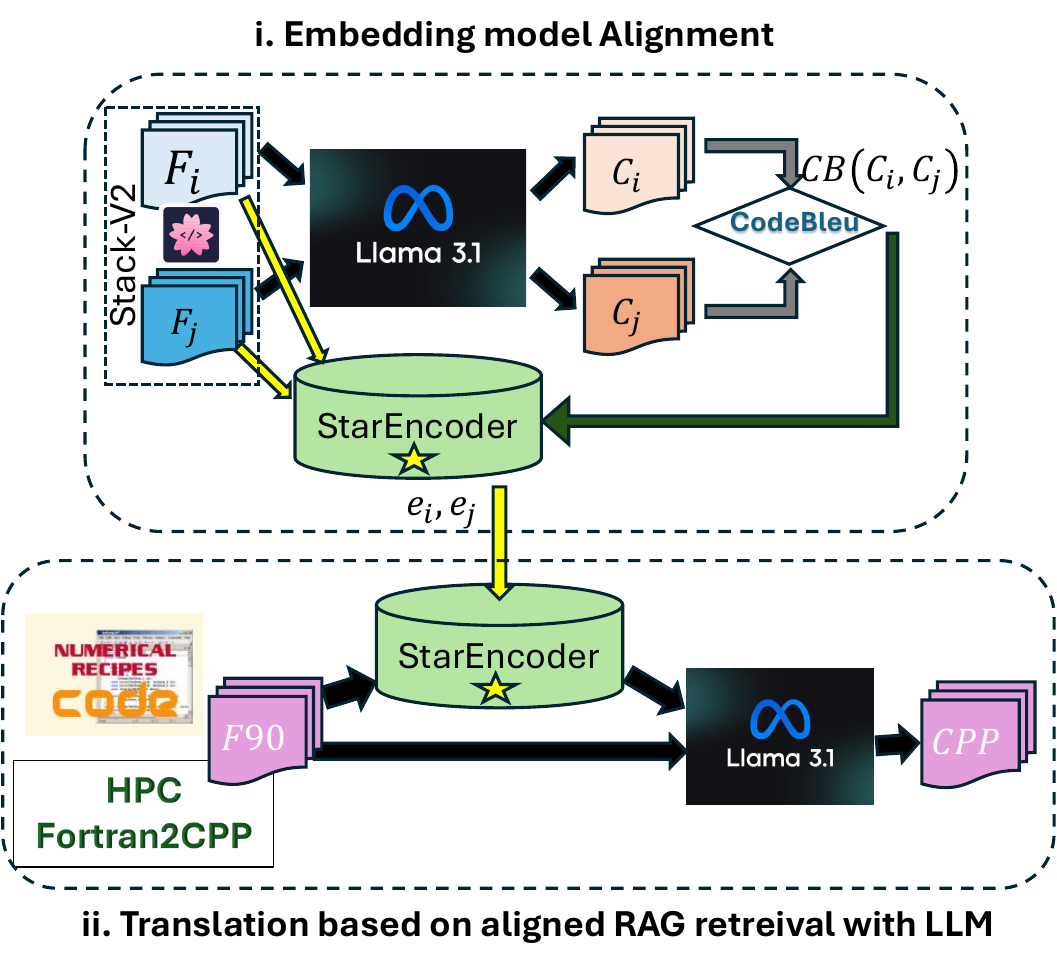}
  \caption{Overview of the proposed pipeline. i) The LLM generates pairwise code translations, which are evaluated using the CodeBLEU metric. ii) The resulting similarity scores are used to guide contrastive learning for semantic alignment of the embedding model.} \label{fig:overview}
\end{figure*}

\subsection{Task-Specific Embedding Alignment}

Our method involves aligning the Fortran embedding model $\Psi$ using contrastive learning based on CodeBLEU similarity scores, followed by applying this aligned model within a RAG framework for improved cross-language code translation from Fortran to C++, as shown in Figure~\ref{fig:overview}I.

\textbf{Embedding Similarity:} Given a pre-trained embedding module $\Phi$, we directly leverage the CodeBLEU similarity computed from the language model $G$ to train an aligned embedding module $\Psi$ for the downstream code translation task. The following discusses how to extract the CodeBLEU similarity from $G$.

From a source dataset of Fortran code snippets $\mathcal{D}^F = \{ c^s_i \}_{i=1}^N$, we generate the corresponding C++ translations $\mathcal{D}^C = \{ c^t_i \}_{i=1}^N$ using $G$ without RAG retrieval:
\begin{equation}
    c^t_i = G(c^s_i), \quad \forall i = 1, \dots, N
    \label{eq:translation}
\end{equation}

Then, we compute the pairwise CodeBLEU similarity scores~\citep{ren2020codebleu} between all generated translation pairs $(c^t_i, c^t_j)$:
\begin{equation}
    S^t_{ij} = \text{CodeBLEU}(c^t_i, c^t_j)
    \label{eq:codebleu_score}
\end{equation}
where the CodeBLEU score matrix $S^t \in [0, 1]^{N \times N}$ is a weighted linear combination of four components: the n-gram match $S_{\text{n-gram}}$, the weighted n-gram match $S_{\text{w-n-gram}}$, the syntactic AST match $S_{\text{syntax}}$, and the semantic data flow match $S_{\text{semantic}}$. These components capture the syntactic and semantic similarities between the generated C++ translations:

\begin{itemize}
    \item $S_{\text{n-gram}}$ is the traditional BLEU score up to n-grams.
    \item $S_{\text{w-n-gram}}$ assigns weights to n-grams based on their importance.
    \item $S_{\text{syntax}}$ measures the similarity between the abstract syntax trees (AST) of the code snippets.
    \item $S_{\text{semantic}}$ assesses the similarity in data flow between code snippets.
\end{itemize}

Intuitively, a high value of $S^t_{ij}$ indicates that the source code snippets $c^s_i$ and $c^s_j$, when translated by $G$, produce similar target code, suggesting that $c^s_i$ and $c^s_j$ are semantically similar with respect to the translation task. Therefore, our approach aims to learn a fine-tuned embedding module $\Psi$ that utilizes $S^t_{ij}$ to enhance code embedding alignment. The approach is expected to guide $\Psi$ in a way that enhances the code translation task leveraging $G$.

\textbf{Embedding Alignment:}

To align the embedding space of code snippets with the semantic similarities measured by CodeBLEU, we propose the Soft Information Noise-Contrastive Estimation (S-InfoNCE) loss applied to the embeddings resulting from the trainable embedding module $\Psi$. On a high level, our proposed S-InfoNCE can be considered a soft version of the InfoNCE loss proposed for contrastive learning~\citep{oord2018representation}.

Given a batch of $N$ code snippets, we compute their embeddings $c^s_{\Psi_i} = \Psi(c^s_i)$ and then calculate the pairwise cosine similarities between those embeddings, scaled by a temperature parameter $\tau > 0$:
\begin{equation}
    S^s_{\Psi_{ij}} = \frac{1}{\tau} \frac{ c^s_{\Psi_i} \cdot c^s_{\Psi_j} }{ \| c^s_{\Psi_i} \| \| c^s_{\Psi_j} \| }
    \label{eq:cosine_similarity}
\end{equation}

Our proposed S-InfoNCE loss integrates these continuous similarity scores to weigh the contribution of each pair. Specifically, the loss component between code $i$ with respect to code $j$ is given as:
\begin{equation}
    l_{ij}^{\textup{S-InfoNCE}}(\Psi) = - S^t_{ij} \log \left( \frac{\exp(S^s_{\Psi_{ij}})}{\sum_{k=1}^N \exp(S^s_{\Psi_{ik}})} \right)
    \label{eq:info_nce_loss}
\end{equation}
and the S-InfoNCE loss is the sum over all code pairs:
\begin{equation}
    \mathcal{L}^{\textup{S-InfoNCE}}(\Psi) = \sum_{i = 1}^N \sum_{j = 1}^N l_{ij}^{\textup{S-InfoNCE}}(\Psi)
    \label{eq:loss_soft_sum}
\end{equation}
Finally, the embedding $\Psi$ is optimized by minimizing $\mathcal{L}^{\textup{S-InfoNCE}}(\Psi)$ using gradient descent.

Compared to the conventional InfoNCE loss for contrastive learning~\citep{oord2018representation}, our proposed loss differs in its usage of $S^t_{ij}$ as a soft indicator for encoding a continuous similarity between the pair $(i, j)$, rather than a binary indicator of class membership (same class or not). This gives rise to the term \textit{soft} InfoNCE, or S-InfoNCE. In the typical InfoNCE loss, the term $l_{ij}$ is included only if the pair $(i, j)$ belongs to the same class, assuming discrete classes are available. However, since such discrete class labels do not exist in the code translation task, we adopt $S^t_{ij}$ as a soft version of this indicator function, allowing for a more nuanced representation of similarity between code pairs.

\begin{lemma}
The stationary points of the S-InfoNCE loss (Equation~\ref{eq:loss_soft_sum}) satisfy:
\begin{equation}
    \frac{\exp(S^s_{\Psi^*_{ij}})}{\sum_{k=1}^N \exp(S^s_{\Psi^*_{ik}})} = \frac{S^t_{ij}}{\sum_{k=1}^N S^t_{ik}},
    \label{eq:stationary_condition}
\end{equation}
for all \( i, j \in \{1, \dots, N\} \).

Furthermore, the optimal loss is the weighted sum of the entropy of the CodeBLEU similarity distribution for each input code \( i \):
\begin{equation}
    \mathcal{L}^{\textup{S-InfoNCE}}(\Psi^*) = \sum_{i=1}^N \left( \sum_{k=1}^N S^t_{ik} \right) H(\boldsymbol{p}^t_i),
    \label{eq:optimal_loss}
\end{equation}
where \( H \) is the entropy function and \( \boldsymbol{p}^t_i \) is a probability vector whose \( j \)-th component is
\begin{equation}
    p^t_{ij} = \frac{S^t_{ij}}{\sum_{k=1}^N S^t_{ik}}.
    \label{eq:ptij}
\end{equation}
\end{lemma}

\begin{proof}
For brevity, let us define:
\begin{itemize}
    \item \( \alpha_{ij} = S^t_{ij} \): the CodeBLEU similarity between the target code translations \( c^t_i \) and \( c^t_j \).
    \item \( p_{ij}(\Psi) = \frac{\exp(S^s_{\Psi_{ij}})}{Z_i} \), where \( Z_i = \sum_{k=1}^N \exp(S^s_{\Psi_{ik}}) \): the normalized exponential of the cosine similarity between the embeddings of source code snippets \( c^s_i \) and \( c^s_j \).
\end{itemize}

The S-InfoNCE loss can be rewritten as:
\begin{equation}
    \mathcal{L}^{\textup{S-InfoNCE}}(\Psi) = - \sum_{i=1}^N \sum_{j=1}^N \alpha_{ij} \log p_{ij}(\Psi).
    \label{eq:loss_rewritten}
\end{equation}

Our goal is to minimize \( \mathcal{L}^{\textup{S-InfoNCE}}(\Psi) \) with respect to \( \Psi \). This can be viewed as a constrained optimization problem over the variables \( p_{ij}(\Psi) \), subject to the normalization constraints:
\begin{equation}
    \sum_{j=1}^N p_{ij}(\Psi) = 1, \quad \forall i \in \{1, \dots, N\}.
\end{equation}

We formulate the Lagrangian \( \mathcal{L} \) as:
\begin{equation}
    \mathcal{L} = - \sum_{i=1}^N \sum_{j=1}^N \alpha_{ij} \log p_{ij}(\Psi) + \sum_{i=1}^N \lambda_i \left( \sum_{j=1}^N p_{ij}(\Psi) - 1 \right).
    \label{eq:lagrangian}
\end{equation}

To find the stationary points, we take the derivative of \( \mathcal{L} \) with respect to \( p_{ij}(\Psi) \) and set it to zero:
\begin{equation}
    \frac{\partial \mathcal{L}}{\partial p_{ij}(\Psi)} = - \frac{\alpha_{ij}}{p_{ij}(\Psi)} + \lambda_i = 0.
    \label{eq:derivative}
\end{equation}

Solving for \( p_{ij}(\Psi) \), we get:
\begin{equation}
    p_{ij}(\Psi) = \frac{\alpha_{ij}}{\lambda_i}.
    \label{eq:pij_solution}
\end{equation}

Applying the normalization constraint:
\begin{align}
    \sum_{j=1}^N p_{ij}(\Psi) &= \sum_{j=1}^N \frac{\alpha_{ij}}{\lambda_i} = 1, \notag \\
    \lambda_i &= \sum_{j=1}^N \alpha_{ij}.
    \label{eq:lambda_solution}
\end{align}

Substituting \( \lambda_i \) back into \( p_{ij}(\Psi) \), we obtain:
\begin{equation}
    p_{ij}(\Psi^*) = \frac{\alpha_{ij}}{\sum_{k=1}^N \alpha_{ik}} = \frac{S^t_{ij}}{\sum_{k=1}^N S^t_{ik}}.
    \label{eq:pij_final}
\end{equation}

This establishes the stationary condition in Equation~\eqref{eq:stationary_condition}.

\textbf{Optimal Loss Calculation:}

Substituting \( p_{ij}(\Psi^*) \) back into the loss function:
\begin{align}
    \mathcal{L}^{\textup{S-InfoNCE}}(\Psi^*) &= - \sum_{i=1}^N \sum_{j=1}^N \alpha_{ij} \log \left( \frac{\alpha_{ij}}{\sum_{k=1}^N \alpha_{ik}} \right) \notag \\
    &= - \sum_{i=1}^N \left( \sum_{j=1}^N \alpha_{ij} \log \alpha_{ij} \right) \notag \\
    & \quad + \sum_{i=1}^N \left( \sum_{j=1}^N \alpha_{ij} \right) \log \left( \sum_{k=1}^N \alpha_{ik} \right).
    \label{eq:loss_optimal}
\end{align}

Let \( A_i = \sum_{k=1}^N \alpha_{ik} \) and \( p^t_{ij} = \dfrac{\alpha_{ij}}{A_i} \). Then, the loss simplifies to:
\begin{align}
    \mathcal{L}^{\textup{S-InfoNCE}}(\Psi^*) &= - \sum_{i=1}^N \left( A_i \sum_{j=1}^N p^t_{ij} \log p^t_{ij} \right) \notag \\
    &= \sum_{i=1}^N A_i H(\boldsymbol{p}^t_i),
\end{align}

where \( H(\boldsymbol{p}^t_i) = - \sum_{j=1}^N p^t_{ij} \log p^t_{ij} \) is the entropy of the probability distribution \( \boldsymbol{p}^t_i \).

Therefore, the optimal loss is:
\begin{equation}
    \mathcal{L}^{\textup{S-InfoNCE}}(\Psi^*) = \sum_{i=1}^N \left( \sum_{k=1}^N S^t_{ik} \right) H(\boldsymbol{p}^t_i).
\end{equation}
\end{proof}

From the lemma, we can see that minimizing the S-InfoNCE loss encourages embeddings of semantically similar code snippets, i.e., those with higher target CodeBLEU score $S^t_{ij}$, to have higher cosine similarities $S^s_{\Psi_{ij}}$, thereby aligning them closer in the embedding space. The temperature parameter \( \tau \) controls the concentration of the distribution: a lower \( \tau \) sharpens the softmax distribution, making the embedding model focus more on the most similar pairs.

\textbf{Retrieval-Augmented Generation with Aligned Embeddings:}

After aligning the embedding model \( \Psi \), we integrate it into theRAG framework to enhance the translation process (Figure~\ref{fig:overview}II). Given a query Fortran code snippet \( c^s_q \), we compute its embedding:

\begin{equation}
    c^s_{\Psi_q} = \Psi(c^s_q).
\end{equation}

We retrieve the top-\( k \) Fortran code snippets \( \{ c^s_{r_1}, c^s_{r_2}, \ldots, c^s_{r_k} \} \) from the corpus \( \mathcal{C} \) by maximizing the cosine similarity between embeddings:

\begin{equation}
    c^s_{r_j} = \arg\max_{c^s \in \mathcal{C}} \text{sim}\left( c^s_{\Psi_q}, \Psi(c^s) \right), \quad j = 1, 2, \ldots, k.
\end{equation}

The corresponding C++ translations \( \{ c^t_{r_1}, c^t_{r_2}, \ldots, c^t_{r_k} \} \) are retrieved alongside the source code snippets.

These retrieved pairs \( \{ (c^s_{r_j}, c^t_{r_j}) \}_{j=1}^k \) are used to augment the input to the language model \( G \), providing additional context:

\begin{equation}
    \hat{c}^t_q = G\left( c^s_q, \{ (c^s_{r_j}, c^t_{r_j}) \}_{j=1}^k \right).
\end{equation}

By incorporating the optimized embedding function \( \Psi \) into the RAG setup, we enhance the performance of the language model without the need for fine-tuning. The retrieval mechanism now provides more relevant examples that are closely aligned with the translation task, leading to more accurate and aligned translations as demonstrated in Appendix~\ref{appdx}.


\section{Experiments and Results}
In our study, we utilized three datasets to enhance code translation through RAG and embedding alignment. The HPC Fortran2CPP dataset~\citep{lei2023creating}, comprising 315 Fortran-C++ code pairs, and the Numerical Recipes dataset~\citep{press1988numerical}, containing 298 Fortran-C++ pairs, were employed for RAG retrieval and evaluation with LLMs. Additionally, we used the Stack-V2 dataset~\citep{lozhkov2024starcoder}, which includes over 500,000 Fortran code snippets, for RAG alignment. From Stack-V2, we sampled 25,000 high-quality and diverse Fortran code snippets by selecting files larger than 500 bytes and prioritizing those with the highest combined star and fork counts, indicating relevance and popularity. Since Stack-V2 lacks Fortran-C++ pairs, we extracted files containing metadata, code, and comments, and utilized the Llama 3.1-70B Instruct model to extract executable Fortran code, discarding other metadata.
\begin{figure*}[!h]
    \centering
    \includegraphics[width=.9\linewidth]{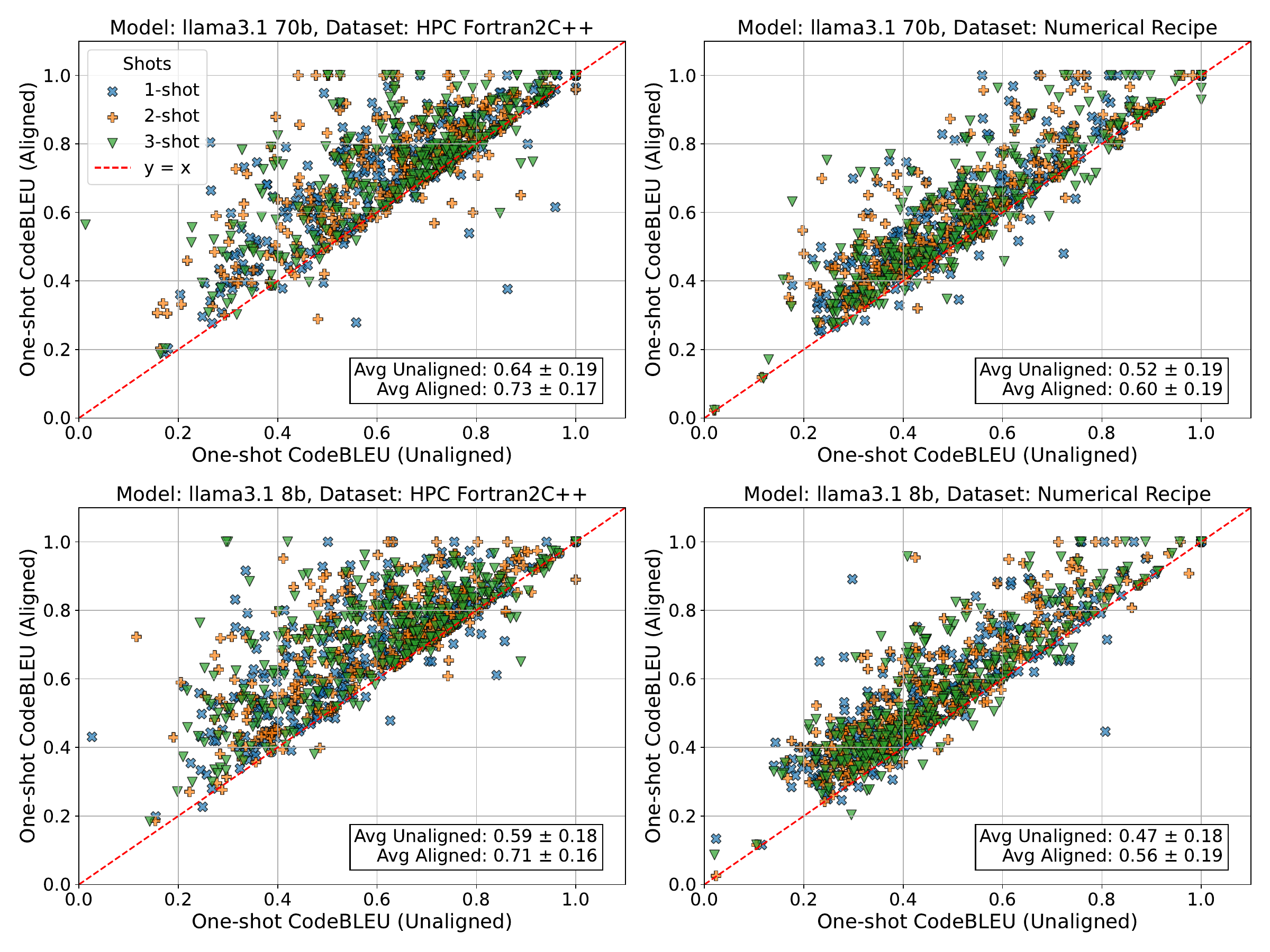}
\caption{Scatter plots comparing the unaligned and aligned One-shot CodeBLEU scores across different shot counts (1-shot, 2-shot, 3-shot) for two models (llama3.1 70b and llama3.1 8b) and two datasets (Numerical Recipe and HPC Fortran2C++ Dataset). Each point represents a shot count, and the red dashed line represents the reference where the unaligned and aligned scores are equal. The text box in each subplot displays the average CodeBLEU performance and standard deviation for aligned vs. unaligned RAG translation across the few-shot configurations.}

    \label{fig:scatter}
\end{figure*}
We selected the StarCoder model~\citep{li2023starcoder} as the embedding backbone for our RAG pipeline and aligned it using contrastive learning on the Stack-V2 dataset. Initially, we use the LLaMA 3.1-8B model to translate the cleaned Fortran code snippets into corresponding C++ code. After code translaton, we computed pairwise CodeBLEU scores between the generated C++ code snippets to quantify the syntactic and semantic similarities of their translations. Leveraging these CodeBLEU metrics and the embeddings from the Fortran codes, we employed the proposed Soft-InfoNCE loss function with a temperature of 0.1 to align the embeddings, effectively training the embedding model to map semantically similar code snippets closer in the embedding space. 

The embedding model was trained using the Adam optimizer with a learning rate of $10^{-3}$ and a batch size of 128 per GPU, sampling approximately 1,280,000 code pairs for alignment. This training process was distributed across 256 GH200 GPUs to accelerate the process, though it can also be performed on fewer GPUs at a significantly slower pace. After alignment, we integrated the embedding model into the RAG pipeline, storing Fortran-C++ pairs along with their Fortran embeddings in a vector database. We then evaluated the performance using the LLaMA 3.1-8B, LLaMA 3.1-70B, Mistral123B, and Mixtral 8x22B models—all instruct-tuned—under zero-shot, 1-shot, 2-shot, and 3-shot settings. The evaluation was conducted on the benchmark datasets HPC Fortran2C++ and Numerical Recipes, following the setup described by~\citep{bhattarai2024enhancing}. The CodeBLEU scores for both the aligned and unaligned models were obtained by comparing the RAG-augmented generated C++ translations against the ground truth C++ code.

  \begin{figure}[!h]
      \centering
      \includegraphics[width=1\linewidth]{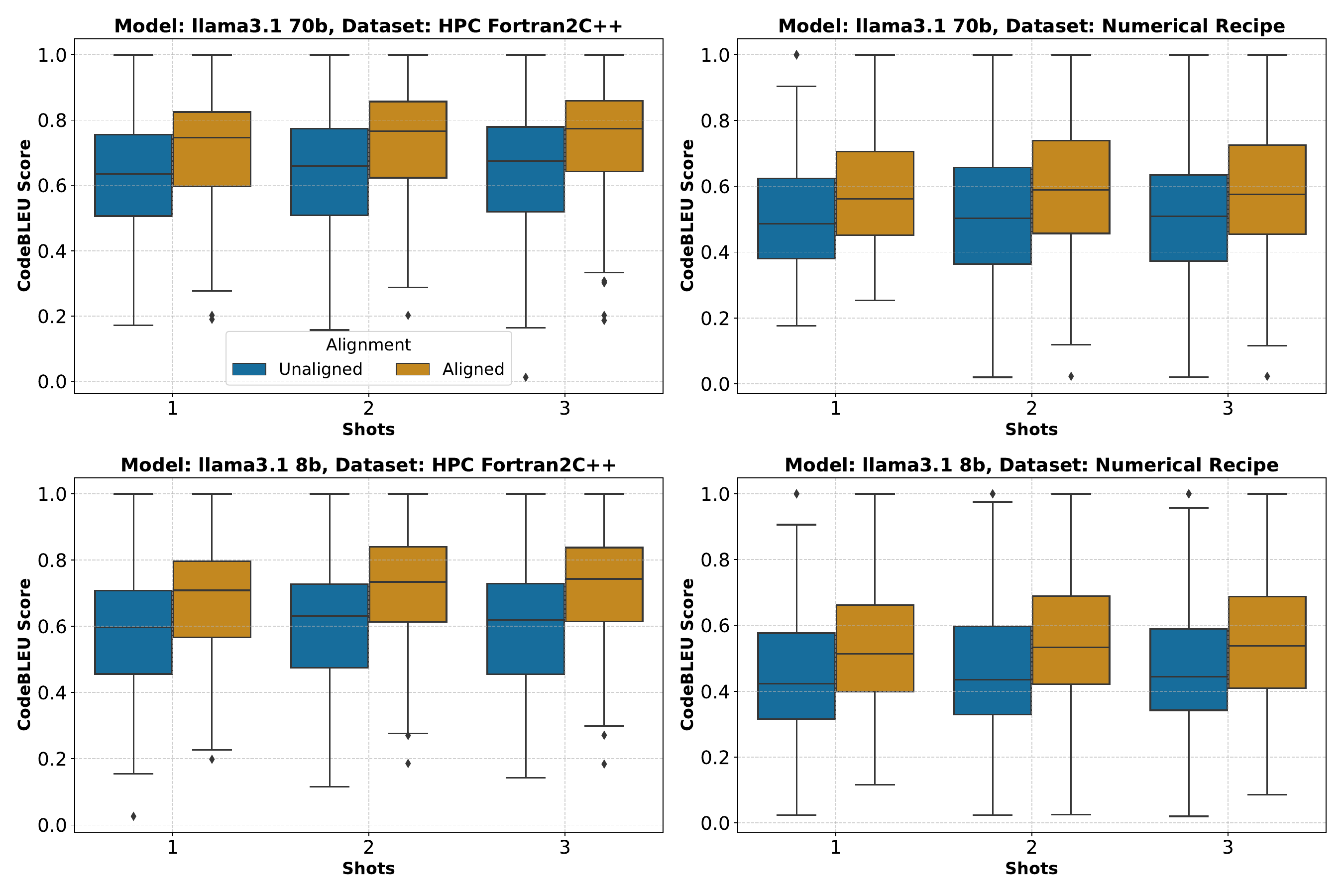}
      \caption{Box plots illustrating the distribution of CodeBLEU scores across various shot counts (1-shot, 2-shot, 3-shot) for both unaligned and aligned models. The results are presented for two models (llama3.1 70b and llama3.1 8b) across two datasets (Numerical Recipe and HPC Fortran2C++ Dataset)}
      \label{fig:boxplot}
  \end{figure}

\begin{table*}[!ht]
\caption{Delta in Mean CodeBLEU scores between Zero- and Few-Shot prompts. The values are presented as Unaligned/Aligned scores.}
\label{tab:mean_codebleu_deltas}

\centering
\begin{tabular}{llllll}
\toprule
& & \multicolumn{4}{c}{\textbf{$\Delta$ in CodeBLEU scores (Unaligned / Aligned)}} \\ 
\cmidrule{3-6}
\textbf{Dataset} & \textbf{Model} & \textbf{Zero-shot} & \textbf{1-shot} & \textbf{2-shot} & \textbf{3-shot} \\ \midrule
HPC Fortran2++ & llama3.1 70b & 0.364 & +0.262/+0.346 & +0.275/+0.371 & +0.281/+0.377 \\
 & llama3.1 8b & 0.342 & +0.237/+0.346 & +0.261/+0.376 & +0.252/+0.374 \\
 & mistral123b & 0.367 & +0.197/+0.241 & +0.210/+0.265 & +0.215/+0.271 \\
  & mixtral-8x22b & 0.376 & +0.237/+0.273 & +0.261/+0.344 & +0.233/+0.304 \\
 \toprule
numerical\_receipe & llama3.1 70b & 0.280 & +0.232/+0.313 & +0.243/+0.329 & +0.243/+0.317 \\
 & llama3.1 8b & 0.276 & +0.181/+0.268 & +0.195/+0.292 & +0.201/+0.289 \\
 & mistral123b & 0.281 & +0.138/+0.169 & +0.132/+0.183 & +0.135/+0.211 \\
  & mixtral-8x22b & 0.280 & +0.200/+0.245 & +0.228/+0.296 & +0.232/+0.312 \\
  
\bottomrule
\end{tabular}
\end{table*}

Figure~\ref{fig:scatter} shows scatter-plots of CodeBLEU scores for code samples produced using RAG retrieval with aligned versus unaligned embeddings derived from StarEncoder. Symbols crosses, pluses and triangles  respectively indicate whether the sample was evaluated using a 1-shot, 2-shot, or 3-shot method. The red dashed lines delineates the boundary where the aligned samples have the same CodeBLEU score as the non-aligned ones, and across all four tested datasets, we observed a  majority of samples above the red line, indicating that the aligned model produces translated codes closer to ground truth. In other words, the results in Figure~\ref{fig:scatter} demonstrate that aligned embeddings significantly improve translation quality for each Fortran-to-C++ code translation task. Specifically, on the HPC Fortran2C++ dataset, averaged over all shot counts and models, the aligned embeddings achieved an average CodeBLEU score of 0.73, whereas unaligned embeddings achieve 0.64. On the Numerical Recipes dataset, aligned embeddings yielded an average CodeBLEU score of 0.60, outperforming the unaligned case at 0.52. These substantial improvements highlight the effectiveness of our method in enhancing translation accuracy.

Figure~\ref{fig:boxplot} further corroborates these findings by presenting the distribution of CodeBLEU scores across various experimental configurations. The box plots reveal that aligned embeddings not only increase the median scores but also reduce performance variability. This indicates that our approach consistently enhances translation quality and leads to more reliable code translations. The consistent improvements across different model sizes (8B and 70B parameters) and datasets demonstrate the robustness and scalability of our method.

Table~\ref{tab:mean_codebleu_deltas} presents the mean CodeBLEU scores for zero-shot and few-shot prompting strategies using both unaligned and aligned embedding models across different language models and datasets.
 A key observation is that the aligned embedding models consistently achieve higher CodeBLEU scores compared to unaligned models when transitioning from zero-shot to few-shot settings.
 For instance, on the HPC Fortran2C++ dataset with the \texttt{LLaMA3.1 70B} model, the aligned model improves from 0.364 to 0.710 (+0.346) in the 1-shot setting, surpassing the unaligned model's improvement from 0.364 to 0.626 (+0.262). Similar trends are observed with the \texttt{LLaMA3.1 8B} model, where the aligned model increases from 0.342 to 0.688 (+0.346), compared to the unaligned model's increase from 0.342 to 0.579 (+0.237). The \texttt{Mistral 13B} and \texttt{Mixtral 8x22B} models also exhibit greater improvements with aligned embeddings in few-shot settings, confirming the benefit of embedding alignment across different architectures.

On the Numerical Recipes dataset, the aligned models again demonstrate superior improvements over unaligned models. For example, the \texttt{LLaMA3.1 70B} aligned model improves from 0.280 to 0.593 (+0.313) in the 1-shot setting, exceeding the unaligned model's increase from 0.280 to 0.512 (+0.232). This consistent pattern across datasets reinforces the advantage of embedding alignment in enhancing code translation performance.

These results indicate that embedding alignment significantly enhances the models' capacity to exploit few-shot prompts, leading to superior code translation performance as measured by CodeBLEU scores. Alignment optimizes the embedding space to better capture the syntactic and semantic nuances of code translation tasks, thereby augmenting the models' few-shot learning capabilities.
Additionally, larger models tend to outperform smaller ones. The \texttt{LLaMA3.1 70B} model consistently achieves higher CodeBLEU scores than the \texttt{LLaMA3.1 8B} model across both datasets and embedding types. The strong performance of the \texttt{Mixtral 8x22B} model, which combines multiple experts, highlights the benefits of increased model capacity.
Furthermore, diminishing marginal gains are observed when increasing the number of shots beyond two, suggesting that the majority of performance improvements are realized with just one or two examples. This indicates that while few-shot examples are beneficial, adding more beyond a certain point yields limited additional gains.

\section{Conclusion}
We introduced a novel method for enhancing cross-language code translation from Fortran to C++ by aligning embeddings within a RAG  framework. By leveraging contrastive learning based on CodeBLEU similarity scores, we aligned the Fortran embedding model so that code snippets yielding high-quality translations are positioned closer in the embedding space. This alignment enables the RAG system to retrieve semantically meaningful examples that effectively guide th LLM during code generation. Our experimental results demonstrate substantial improvements in translation quality without the need for fine-tuning the LLM. Specifically, using aligned embeddings increased the average CodeBLEU score from 0.64 to 0.73 on the HPC Fortran2C++ dataset and from 0.52 to 0.60 on the Numerical Recipes dataset, representing relative improvements of approximately $14\%$ and $15\%$, respectively. The larger model (\texttt{llama3.1 70b}) consistently outperformed the smaller model (\texttt{llama3.1 8b}), indicating that increased model capacity enhances the effectiveness of our approach. Additionally, we observed diminishing returns beyond two-shot prompting, suggesting that most performance gains are achieved with just one or two examples.
Thus, our approach significantly improves code translation performance by optimizing the retrieval mechanism through task-specific embedding alignment, rather than relying on computationally expensive fine-tuning of the LLM. This method is computationally efficient, scalable, and adaptable to other code translation tasks, particularly when aligned datasets are scarce or evaluation metrics like CodeBLEU are critical. Future work could extend this alignment strategy to additional programming languages and explore integrating other evaluation metrics to further enhance translation quality.

\section{Limitations}
Our approach leverages CodeBLEU as a task-specific metric for performing contrastive learning via a custom Soft-InfoNCE loss in the alignment of embedding models for code translation. While this approach introduces several improvements, it also brings specific limitations.
First, using CodeBLEU as the basis for contrastive learning focuses primarily on syntactic and semantic alignment, which may not always translate into functional equivalence. CodeBLEU, while effective at evaluating linguistic features of generated code, does not fully capture the functional behavior of code, meaning that two semantically similar snippets could still behave differently at runtime~\citep{ren2020codebleu}. This limitation can lead to cases where the retrieval mechanism selects semantically similar but functionally incorrect examples, impacting the overall quality of the translation task.
Second, contrastive learning, particularly with InfoNCE loss, relies heavily on the assumption that maximizing the similarity between pairs (based on CodeBLEU) leads to better downstream performance. However, InfoNCE loss is limited by its focus on pulling positive samples closer while pushing away negative ones, which in the case of code translation, does not always capture the subtle nuances of code equivalence across languages ~\citep{khosla2020supervised}. Code snippets with different syntactic structures but similar functionality may be treated as negative examples, leading to a misaligned embedding space and suboptimal retrieval.
Third, the granularity of the CodeBLEU score presents an inherent challenge. Since CodeBLEU provides a continuous similarity metric (between 0 and 1), aligning embeddings through InfoNCE loss may not fully capture the wide range of functional similarities or dissimilarities between code snippets. This results in an embedding space that reflects linguistic rather than purely functional similarity, which can lead to errors in retrieval when applied to real-world translation tasks where functional correctness is paramount ~\citep{codebert}.
Additionally, the use of CodeBLEU as a basis for contrastive learning is highly dependent on the quality of the generated code samples and their reference translations. Any noise or imperfections in the training data (e.g., low-quality code or inconsistent style) may degrade the alignment process. Since InfoNCE relies on subtle positive and negative distinctions, noisy CodeBLEU scores can introduce ambiguity, further distorting the learning process and leading to poorer retrievals during generation~\cite{wang2021understanding}.


\bibliography{main}
\bibliographystyle{acl_natbib}
\appendix
\section{Demonstration: Enhancing Fortran-to-C++ Translation Using Aligned RAG on LLaMa3.1 70b Model}
\label{appdx}
\subsection{Fortran Code to be Translated}
Below is the original Fortran code that needs to be translated into C++:

\begin{lstlisting}[language=Fortran,caption={Fortran Code}]
program DRB011_minusminus_orig_yes
    use omp_lib
    implicit none

    integer :: i, len, numNodes, numNodes2
    integer :: x(100)

    len = 100
    numNodes = len
    numNodes2 = 0

    do i = 1, len
        if (MOD(i,2) == 0) then
            x(i) = 5
        else
            x(i) = -5
        end if
    end do

    !$omp parallel do
    do i = numNodes, 1, -1
        if (x(i) <= 0) then
            numNodes2 = numNodes2 - 1
        end if
    end do
    !$omp end parallel do

    print*, "numNodes2 =", numNodes2
end program
\end{lstlisting}

\subsection{Ground Truth}
Below is the human-generated ground truth Fortran-to-C++ translation, which serves as the baseline:

\begin{lstlisting}[language=C++,caption={Ground Truth}]
#include <stdlib.h>
#include <stdio.h>
int main(int argc, char* argv[]) {
    int i;
    int len = 100;
    int numNodes = len;
    int numNodes2 = 0;
    int x[100];

    // initialize x[]
    for (i = 0; i < len; i++) {
        if (i % 2 == 0)
            x[i] = 5;
        else
            x[i] = -5;
    }

#pragma omp parallel for
    for (i = numNodes - 1; i > -1; --i) {
        if (x[i] <= 0) {
            numNodes2--;
        }
    }

    printf("numNodes2 = %d\n", numNodes2);
    return 0;
}
\end{lstlisting}

\subsection{Zero-shot Translation}
The zero-shot output from the LLM demonstrates basic translation abilities but with structural issues:

\begin{lstlisting}[language=C++,caption={Zero-shot Translation}]
#include <iostream>
#include <omp.h>

int main() {
    const int len = 100;
    int numNodes = len;
    int numNodes2 = 0;
    int x[len];

    for (int i = 1; i <= len; i++) {
        if (i % 2 == 0) {
            x[i - 1] = 5;
        } else {
            x[i - 1] = -5;
        }
    }

    #pragma omp parallel for reduction(-:numNodes2)
    for (int i = numNodes; i >= 1; i--) {
        if (x[i - 1] <= 0) {
            numNodes2--;
        }
    }

    std::cout << "numNodes2 = " << numNodes2 << std::endl;

    return 0;
}
\end{lstlisting}

\subsection{Unaligned RAG One-shot Translation}
The one-shot output improves the translation, but some structural issues still remain. Specifically, the initialization of variables and handling of array indices are slightly different:

\begin{lstlisting}[language=C++,caption={One-shot Translation}]
#include <stdio.h>

int main() {
    int i, len, numNodes, numNodes2;
    int x[100];

    len = 100;
    numNodes = len;
    numNodes2 = 0;

    for (i = 0; i < len; i++) {
        if ((i + 1) % 2 == 0) {
            x[i] = 5;
        } else {
            x[i] = -5;
        }
    }

#pragma omp parallel for reduction(-:numNodes2)
    for (i = numNodes - 1; i >= 0; i--) {
        if (x[i] <= 0) {
            numNodes2--;
        }
    }

    printf("numNodes2 = %d\n", numNodes2);

    return 0;
}
\end{lstlisting}

\subsection{Aligned RAG Translation}
By aligning the retrieval process to the context of the task, the translation becomes much more accurate, closely matching the ground truth:

\begin{lstlisting}[language=C++,caption={Aligned RAG-based Translation}]
#include <stdio.h>
#include <omp.h>

int main(int argc, char* argv[]) {
    int i;
    int len = 100;
    int numNodes = len;
    int numNodes2 = 0;
    int x[100];

    for (i = 0; i < len; i++) {
        if ((i % 2) == 0) {
            x[i] = 5;
        } else {
            x[i] = -5;
        }
    }

#pragma omp parallel for
    for (i = numNodes - 1; i >= 0; i--) {
        if (x[i] <= 0) {
            numNodes2--;
        }
    }

    printf("numNodes2 = %d\n", numNodes2);
    return 0;
}
\end{lstlisting}
\end{document}